\documentclass{ieeeaccess}
\usepackage{cite}
\usepackage{amsmath,amssymb,amsfonts,amsthm,mathtools}
\usepackage[ruled,vlined]{algorithm2e}
\usepackage{graphicx}
\usepackage{float}
\usepackage{afterpage}
\usepackage[caption=false,font=footnotesize]{subfig}
\usepackage{booktabs}
\usepackage{textcomp}
\usepackage{pifont}
\usepackage{bm}
\usepackage[hidelinks]{hyperref}
\graphicspath{{./}}

\makeatletter
\AtBeginDocument{\DeclareMathVersion{bold}
\SetSymbolFont{operators}{bold}{T1}{times}{b}{n}
\SetSymbolFont{NewLetters}{bold}{T1}{times}{b}{it}
\SetMathAlphabet{\mathrm}{bold}{T1}{times}{b}{n}
\SetMathAlphabet{\mathit}{bold}{T1}{times}{b}{it}
\SetMathAlphabet{\mathbf}{bold}{T1}{times}{b}{n}
\SetMathAlphabet{\mathtt}{bold}{OT1}{pcr}{b}{n}
\SetSymbolFont{symbols}{bold}{OMS}{cmsy}{b}{n}
\renewcommand\boldmath{\@nomath\boldmath\mathversion{bold}}}
\makeatother

\theoremstyle{plain}
\newtheorem{theorem}{Theorem}[section]

\newtheorem{lemma}[theorem]{Lemma}
\newtheorem{corollary}[theorem]{Corollary}
\theoremstyle{definition}

\theoremstyle{remark}

\def\BibTeX{{\rm B\kern-.05em{\sc i\kern-.025em b}\kern-.08em
    T\kern-.1667em\lower.7ex\hbox{E}\kern-.125emX}}

\begin{document}
\history{Date of publication xxxx 00, 0000, date of current version xxxx 00, 0000.}
\doi{10.1109/ACCESS.2024.0429000}

\title{Self-Abstraction Learning for Effective and Stable Training of Deep Neural Networks}
\author{\uppercase{Wonyong Cho}$^{*}$\authorrefmark{1}, \uppercase{Taemin Kim}$^{*}$\authorrefmark{2},
\uppercase{Jungmin Kim}\authorrefmark{1}, \uppercase{Jeong-Rae Kim}\authorrefmark{1}, and
\uppercase{Sung Hoon Jung}\authorrefmark{2}}

\address[1]{Department of Mathematics, University of Seoul, Seoul 02504, South Korea}
\address[2]{Department of Applied Artificial Intelligence, Hansung University, Seoul 02876, South Korea}
\tfootnote{*Wonyong Cho and Taemin Kim contributed equally to this work.\\[0.5\baselineskip]This work was supported in part by the 2023 sabbatical year research grant of the University of Seoul and in part by Hansung University.}

\markboth
{Cho et al.: Self-Abstraction Learning for Effective and Stable Training of Deep Neural Networks}
{Cho et al.: Self-Abstraction Learning for Effective and Stable Training of Deep Neural Networks}

\corresp{Corresponding authors: Jeong-Rae Kim (jrkim@uos.ac.kr) and Sung Hoon Jung (shjung@hansung.ac.kr).}

\begin{abstract}
  Training large-scale deep neural networks effectively and stably is essential for applying deep learning across various fields. However, conventional methods, which rely on training a single large network, often encounter challenges such as gradient vanishing, overfitting and unstable learning. To overcome these limitations, we introduce Self-Abstraction Learning (SAL), a hierarchical framework. In SAL, networks are arranged by structural complexity, where the simplest topmost network is trained first and its hidden and output layers serve as guidance for the successively more complex networks below. This top-down sequential guidance effectively mitigates optimization issues, enabling stable training of deep architectures. Various experiments across MLP, CNN, and RNN architectures demonstrate that SAL consistently outperforms conventional methods, ensuring robust generalization even in data-scarce and complex network regimes.
\end{abstract}


\titlepgskip=-21pt
\maketitle

\section{Introduction}

Deep learning models such as MLP, CNN, and RNN have demonstrated remarkable performance across diverse applications \cite{krizhevsky2012imagenet,sutskever2014sequence,goodfellow2016deep}. However, scaling these models to achieve higher performance typically involves increasing architectural complexity and data scale, which leads to significant optimization problems—such as gradient instability \cite{glorot2010understanding} and severe overfitting. We argue that these challenges arise from attempting to model complex input-output relationships without sufficiently structured intermediate representations \cite{bengio2013representation}.

Thus, we propose a learning approach termed Self-Abstraction Learning (SAL), inspired by the divide-and-conquer strategy widely used in computer engineering. In SAL, a small Artificial Neural Network (ANN) model with fewer layers and nodes is initially trained on a coarse-grained representation space at a higher floor. Subsequently, the values of the hidden and output layers of this higher floor ANN model are utilized as targets for training an ANN model with more layers and nodes at a lower floor, enabling it to learn richer representations. 

This hierarchical process begins with the ANN at the top floor and progresses to the bottom floor, training each floor model for a fixed number of epochs. By following this approach, the ANN model at the lower floor is guided by the knowledge of the abstracted representations captured by the higher floor ANN model. This strategy offers several advantages over conventional learning methods:

\begin{itemize}
\item \textbf{Mitigated Gradient Vanishing}: The step-by-step guidance from a top-floor ANN with minimal hidden layers mitigates the gradient vanishing problem. Additionally, the hierarchical approach reduces the reliance on widely used activation functions such as ReLU, which can cause gradient explosion \cite{he2015delving,goodfellow2016deep}. This characteristic removes the dependency on additional mechanisms such as skip connections \cite{he2016deep} to address gradient vanishing. 
\item \textbf{Effective Optimization}: By progressively refining the abstracted representations, the model can converge to the effective representation space.
\item \textbf{Reduced Overfitting}: By learning from coarse-grained representations, the likelihood of overfitting is significantly reduced.
\item \textbf{Robust Performance with Limited Data}: The method shows resilience against overfitting even when training deep models with a small amount of data.
\end{itemize}

Overall, these advantages imply SAL's potential to alleviate the inherent bias-variance trade-off. Unlike conventional training methods that often struggle between high bias in low-complexity networks and high variance in high-complexity networks, SAL helps reconcile this conflict. It achieves superior performance by leveraging guidance from upper floors to mitigate high variance, while simultaneously preserving the low bias inherent in the complex lower floors.

To validate the effectiveness of the proposed SAL method, we conducted various experiments using MLP, CNN, and LSTM models. The results demonstrated that even when training complex networks with limited data, the SAL method achieved effective learning while suppressing gradient vanishing and overfitting. Particularly, on medical datasets in data-scarce regimes, the CNN trained by SAL demonstrated superior performance across all five evaluation metrics—AUROC, F1-Score, Sensitivity, Specificity, and Accuracy—showing its stability and robustness.

Based on these findings, we confirm that the SAL method exhibits all four effects claimed in this work. In future research, it will be necessary to examine whether these effects also hold in more complex architectures such as Transformers \cite{vaswani2017attention}.

\section{Related Work}
\subsection{Knowledge Distillation}
Knowledge Distillation \cite{hinton2015distilling} transfers knowledge from a large pre-trained teacher network to a smaller student network by training the student to mimic the teacher's soft output distribution. FitNets \cite{romero2015fitnets} extended this by introducing intermediate hint layers, where the student is trained to match the output of a selected hidden layer of the teacher via a regressor. While SAL also involves multiple networks of different scales and propagates signals across hidden layers, its objective and direction are fundamentally different. These distillation methods operate in a complex-to-simple direction for the purpose of model compression, requiring a fully pre-trained teacher before distillation begins. SAL operates in the opposite direction, where structurally simpler 
upper-floor models sequentially guide multiple lower-floor models 
of increasing complexity within a single training pipeline, 
with the goal of stable and effective training of the most 
complex model. The final deployed model in SAL is the most 
complex one, not the compressed student.

\setlength{\textfloatsep}{25pt}
\begin{algorithm}[tb]
   \caption{SAL Algorithm}
   \label{alg:sal}
   \DontPrintSemicolon
   \KwIn{$\mathcal{D}$ (dataset), $S_{\text{max}}$ (steps), $t$ (training epochs), $r$ (guidance epochs)}
   \textbf{Models:} $\{ANN_f\}_{f=1}^F$ ($ANN_F$ is the topmost)\;
   \For{$step = 1$ \KwTo $S_{\text{max}}$}{
      \For{$f = F-1$ \mbox{\normalfont\bfseries\upshape down to} $1$}{
         Train $ANN_{f+1}$ on $\mathcal{D}$ for $t$ epochs (CE loss)\;
         \For{$epoch = 1$ \KwTo $r$}{
            \For{\textnormal{each mini-batch} $d \subset \mathcal{D}$}{
               Store the values of all hidden and output layers of $ANN_{f+1}$\;
               Train $ANN_f$ using stored values of $ANN_{f+1}$ as targets (MSE loss)\;
            }
         }
      }
      Train $ANN_1$ on $\mathcal{D}$ for $t$ epochs (CE loss)\;
   }
\end{algorithm}

\subsection{Simple to Complex}
Curriculum learning \cite{bengio2009curriculum} proposes training models by presenting data in order of increasing difficulty, sharing a conceptual alignment with SAL's progressive approach. ProGAN \cite{karras2018progressive} extended this principle to the model level by incrementally adding layers to stabilize GAN training. Net2Net \cite{chen2016net2net} and Network Morphism \cite{wei2016network} enable rapid capacity expansion through function-preserving weight transformations that initialize larger networks from smaller ones. While these methods share the simple-to-complex philosophy with SAL, they each modulate either the difficulty of training data or the architecture of a single growing model. SAL instead trains structurally separate floor models and explicitly transfers intermediate representations between them. Rather than growing a single model or reordering data, SAL guides the complex lower floor through sequential representation-level targets.

\section{Self-Abstraction Learning}

This section outlines the structure and benefits of the proposed SAL. The SAL algorithm, described in Algorithm~\ref{alg:sal}, operates by progressively increasing the complexity of the model on each floor, starting from the simplest model on the top.
The SAL pipeline is designed to use the training information of each layer as the target for the lower floor model, enabling the model to gradually learn more detailed representations. 

After training the top model using CE (Cross-Entropy) loss, its hidden and output layers' values are stored. As shown in Figure \ref{fig:sal_two_floor}, these stored values serve as targets for training the lower floor model. The lower model is trained for multiple epochs based on the targets generated by the upper model. During this process, the MSE loss is used to minimize the difference between the target values and the layers' outputs of the lower model, enabling the layers to align with the representations learned by the upper model. Subsequently, the model containing the guided layers is trained on $\mathcal{D}$ using CE loss, and its hidden and output layer values then serve as guidance targets for the next lower floor.
After recursively applying this guidance process from the top floor model $ANN_F$ down to the penultimate floor model $ANN_2$, the bottom floor model $ANN_1$ is trained using the CE loss. This entire sequence of processes is repeated over $S_{\max}$ steps.

From this hierarchical structure, each model progressively learns fine-grained representations based on the outputs of the floor above. Consequently, this enables the bottom model to capture the most fine-grained details. The method is visualized in Figure \ref{fig:sal_two_floor} and Figure \ref{fig:sal_three_floor}. Figure \ref{fig:sal_two_floor} illustrates the SAL configuration of a two-floor MLP model, while Figure \ref{fig:sal_three_floor} demonstrates how the hidden and output layers at the upper floors guide the corresponding layers at the lower floors in a three-floor structure.

The pipeline of SAL facilitates stable and effective learning by enabling the lower floors to build upon the simplified representations learned by the upper floors. By progressively refining the representation space, SAL enables superior performance over conventional training, particularly in deep neural networks with limited data. We demonstrate these benefits of our approach through theoretical justification and experimental results.

\begin{figure}[t]
    \centering
    \includegraphics[width=\columnwidth, keepaspectratio]{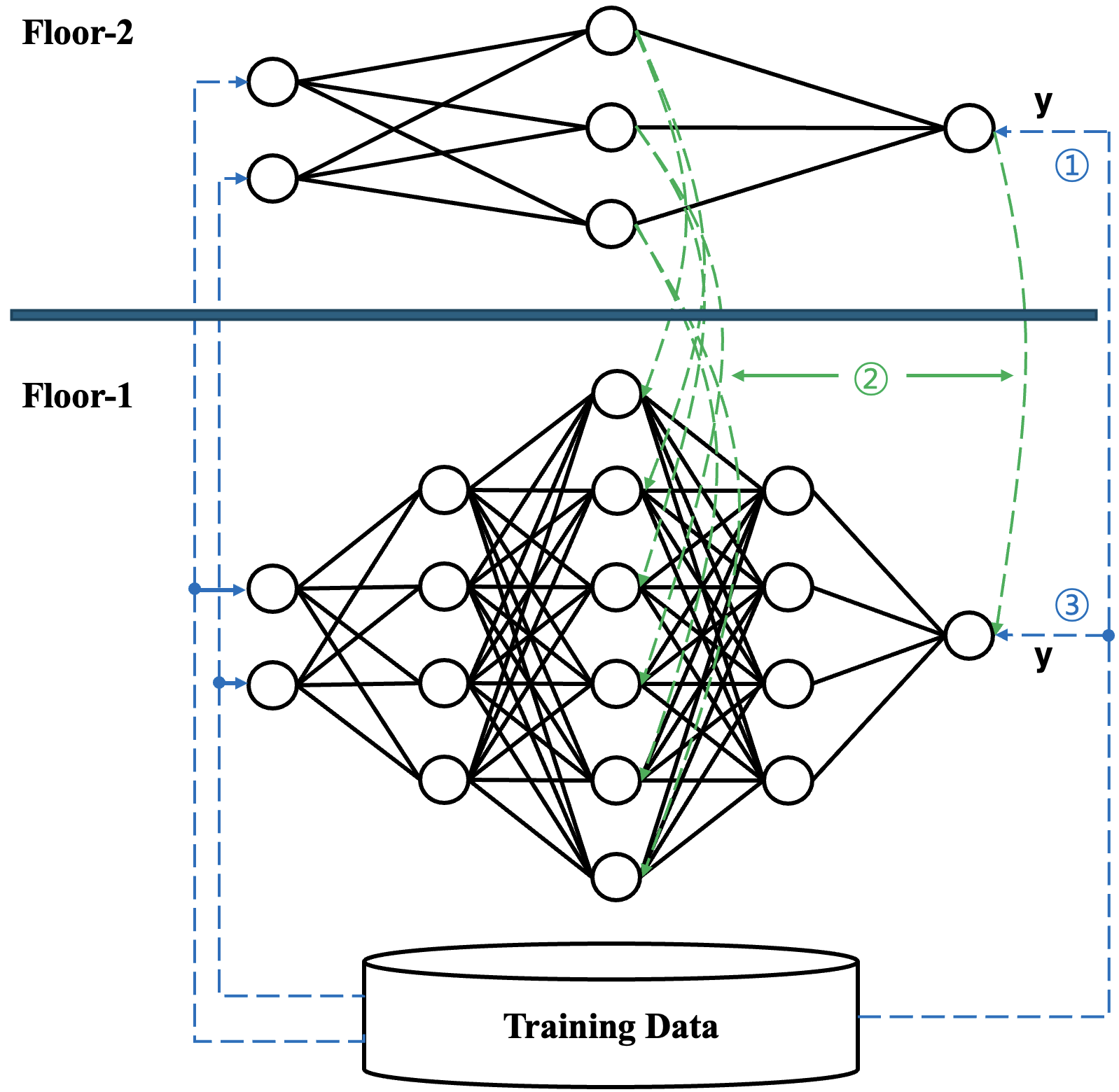}
    \caption{SAL Framework (illustrated with a 1:2 hidden layer width ratio). \ding{172} represents the direct training phase, where Floor-2 is trained for fixed $t$ epochs using labels $y$. \ding{173} indicates the guidance phase, where Floor-1's layers optimize towards Floor-2's values for fixed $r$ epochs. \ding{174} replicates the direct training phase, where Floor-1 is trained for $t$ epochs using labels $y$.} 
    \label{fig:sal_two_floor}
\end{figure}

\section{Theoretical Analysis}
\label{sec:theoretical}
In this section, we provide a theoretical analysis of the generalization guarantees in SAL, emphasizing its superior efficacy in data-scarce regimes. We theoretically show that aligning a high-capacity model with a simpler one allows it to inherit robust generalization properties, thereby mitigating overfitting effectively in such settings.

Let $\mathcal{X}$ be the input data space and $\mathcal{Y} = [0,1]^n$ be the output space. We assume that data points are drawn from a distribution $\mathcal{D}$ over $\mathcal{X} \times \mathcal{Y}$. We now formally demonstrate that the lower, more complex model $f$ effectively inherits the generalization ability of the simpler upper model $h$. For any function $f$, we define the generalization error of $f$ by
$R(f) = \mathbb{E}_{(x,y)\sim \mathcal{D}}\big[l(f(x),y)\big]$
and its empirical loss on a sample $S = \{(x_i, y_i)\}^m_{i=1}$ by
$\hat{R}_S(f) = \frac{1}{m}\sum_{i=1}^{m} l(f(x_i), y_i)$ \cite{mohri2018foundations},
where $l$ denotes the CE loss function combined with the softmax function.
Also, $\hat{\mathcal{R}}_{S}(\mathcal{F})$ denotes the empirical Rademacher complexity of a function class $\mathcal{F}$. Let the function class of the upper floor model $ANN_{k+1}$ be $\mathcal{H}$ and that of the lower floor model $ANN_{k}$ be $\mathcal{F}$. 
Since $ANN_{k+1}$ is simpler than $ANN_k$, we easily obtain $\mathcal{H} \subset \mathcal{F}$.
After sufficient guidance from the upper floor model, we can assume that $\mathbb{E}_{(x, y)\sim \mathcal{D}}[\lVert f(x)-h(x)\rVert_2]$
becomes negligible for $f\in\mathcal{F}$ and $h\in\mathcal{H}$.
\vspace{\baselineskip}

\begin{theorem}[{Generalization Error Bound}]
\label{theorem1}
Let $f \in \mathcal{F}$ and $h \in \mathcal{H}$ be functions satisfying $\mathbb{E}[\|f(x) - h(x)\|_2] < \epsilon$ for a small $\epsilon > 0$.
There exists a constant $C > 0$ such that
\begin{equation*}
R(f) \le R(h) + C \cdot \epsilon.
\end{equation*}
\end{theorem}
This can be proven using the triangle inequality and the Lipschitz continuity of the loss function. A detailed proof can be found in the Appendix \ref{app:proof_of_theorem_4.1}. In the data-scarce regime, where the simpler model $h$ tends to generalize better, this theorem implies that $f$ inherits the strong generalization ability of $h$.

\vspace{\baselineskip}

By the generalization error bound inequality \cite{mohri2018foundations}, $R(h)$ is bounded by $\hat{R}_{S}(h)+2\hat{\mathcal{R}}_{S}(\mathcal{H})+3\sqrt{\frac{\log\frac{2}{\delta}}{2m}}$ with a probability at least $1-\delta$. Combining this with Theorem \ref{theorem1}, we obtain the following Corollary.

\begin{figure}[t]
    \centering
    \includegraphics[width=\columnwidth, keepaspectratio]{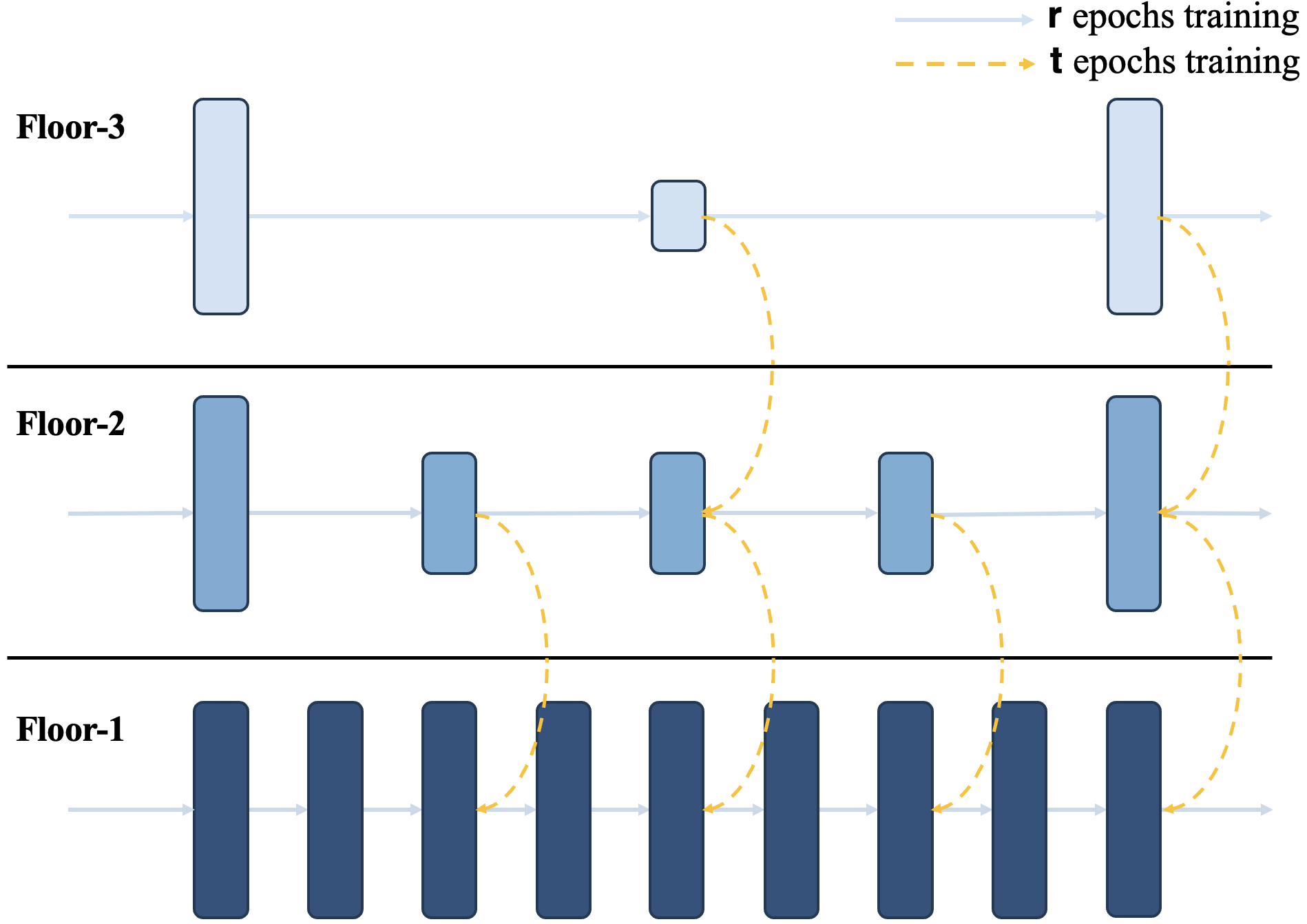}
    \caption{SAL learning process for a three-floor structure. Solid and dashed arrows represent training using CE loss and MSE loss, respectively.}
    \label{fig:sal_three_floor}
\end{figure}

\vspace{\baselineskip}

\begin{corollary} From the perspective of Rademacher complexity, for any $\delta > 0$, the following inequality holds with probability at least $1-\delta$.
\begin{equation*}
R(f) \le \hat{R}_{S}(h)+2\hat{\mathcal{R}}_{S}(\mathcal{H})+3\sqrt{\frac{\log\frac{2}{\delta}}{2m}}+C \cdot \epsilon.
\end{equation*}
\end{corollary}
Since $\mathcal{H}\subset\mathcal{F}$, it follows $\hat{\mathcal{R}}_{S}(\mathcal{H})\le\hat{\mathcal{R}}_{S}(\mathcal{F})$. For sufficiently large $m$, $\hat{R}_{S}(h)$ and $\hat{R}_{S}(f)$ can be
considered similar in value. A detailed discussion on the sample size $m$ can be found in Appendix \ref{app:discussion on sample size}. Consequently, this can yield a tighter upper bound than the generalization error bound \cite{mohri2018foundations} for $f$. This indicates that the generalization error of the more
complex model $f$ is also effectively governed by the Rademacher complexity of the simpler model's function class $\mathcal{H}$.

Based on this theoretical analysis, we conclude that SAL effectively mitigates overfitting, a primary cause of performance degradation in data-scarce regimes. This theoretical bound supports the observation that SAL achieves robust performance in such situations. Collectively, these derivations imply that the top-down guidance acts as a form of implicit regularization. By constraining the optimization trajectory of the complex lower model to align with the simpler upper model, SAL effectively narrows the search space successively.

\section{Experiments}
\subsection{Experiments Setup}
\label{sec:experiments setup}
This study evaluates the proposed characteristics of SAL, including mitigated gradient vanishing, effective optimization, robust performance in limited data environments, and reduced overfitting, as outlined in the introduction. Experiments were designed to compare the performance of SAL with conventional training methods across various neural network architectures and datasets. For the plain baseline network, we utilized a standard architecture, trained without any hierarchical structures. To ensure a fair comparison, the bottom floor (Floor-1) of the SAL network was configured to be identical to this baseline. We fixed the direct training epochs ($t$) at 5 for image datasets, 2 for text datasets, and the guidance epochs ($r$) at 10 for all datasets.

Table \ref{tab:datasets} summarizes the datasets used in our experiments. The MLP models were tested on MNIST \cite{lecun2002gradient}, CIFAR10, CIFAR100 \cite{krizhevsky2009learning}, and IMAGENETTE \cite{deng2009imagenet} datasets. The CNN models were evaluated on the same image datasets, and the LSTM models were tested on CARER \cite{saravia2018carer}, TREC, and TREC\_FINE \cite{li2002learning}. All models are configured with a fixed depth of 8 layers at Floor-1, which
serves as the bottom floor and is identical to the plain baseline. For SAL with floor size $F{=}3$, the upper floors are constructed by halving both depth and width at each successive floor: Floor-2 uses 4 layers at half the width, and Floor-3 uses 2 layers at a quarter of the width. Formally, for floor $F$ counted from the bottom, the depth is
$\lfloor 8 / 2^{F-1} \rfloor$ and the hidden size is $W / 2^{F-1}$,
where $W$ denotes the base hidden size of Floor-1. For MLP, $W$ is set
to 256 for MNIST and CIFAR10, and 512 for CIFAR100 and IMAGENETTE. For
CNN, $W$ refers to the number of channels and is set to 512 across all
image datasets. For LSTM, $W$ is set to 64 for TREC and TREC\_FINE, and
128 for CARER. Within each floor, all layers share the same hidden size, so that floors
differ only in depth and hidden size. This uniform within-floor design
keeps the architectural configuration simple and reproducible across
all experimental settings.

To evaluate the robustness of SAL under data-scarce regimes, experiments were conducted using the full datasets and subsets containing 1/10 and 1/100 of the original data. These are denoted as (1), (1/10), and (1/100) respectively in Table \ref{tab:Gradient_Vanishing}, Table \ref{tab:SAL_Performance}, Table \ref{tab:hyperparam_epoch}, and Table \ref{tab:hyperparam_floor}. Both the plain networks and the SAL networks were trained for a duration sufficient to trigger early stopping, set with a patience of 10 epochs for plain and 10 steps for SAL to maximize each model's performance. All experimental results from
Table~\ref{tab:SAL_Performance} onward are obtained using the ReLU
activation function, and all reported values are mean $\pm$ standard
deviation over five independent runs. We did not employ additional regularization methods such as Batch Normalization, Layer Normalization, or Dropout in any experiment except the BUSI application in Section~\ref{sec:busi}, in order to isolate the effect of the proposed learning framework.

\begin{table}[!t]
\caption{Overview of datasets used in the experiments, which includes description, number of classes, and sizes for image classification and text classification tasks.}
\label{tab:datasets}
\begin{center}
\setlength{\tabcolsep}{4pt}
\resizebox{\columnwidth}{!}{\begin{tabular}{lp{4.5cm}cc} 
\toprule
\textbf{Dataset} & \textbf{Description} & \textbf{Classes} & \textbf{Size (Train / Test)} \\
\midrule
MNIST & Grayscale images of handwritten digits & 10 & 60,000 / 10,000 \\
CIFAR10 & Color images of 10 object categories & 10 & 50,000 / 10,000 \\
CIFAR100 & Color images of 100 object categories & 100 & 50,000 / 10,000 \\
IMAGENETTE & Subset of ImageNet with 10 categories & 10 & 9,469 / 3,925 \\
CARER & English Twitter messages (6 emotions) & 6 & 16,000 / 2,000 \\
TREC & Text data for question classification & 6 & 5,452 / 500 \\
TREC\_FINE & Fine-grained version of TREC & 50 & 5,452 / 500 \\
\bottomrule
\end{tabular}}
\end{center}
\end{table}

\subsection{Gradient Vanishing and Stability}
To verify the efficacy of SAL in mitigating gradient vanishing, we employed the sigmoid function as an activation function, which is susceptible to this problem \cite{goodfellow2016deep}. The LSTM model, which internally utilizes the tanh activation function, was excluded from this analysis. Table \ref{tab:Gradient_Vanishing} demonstrates that SAL successfully mitigates the gradient vanishing and generally yields superior performance compared to plain baseline training. 

As shown in Table \ref{tab:Gradient_Vanishing}, SAL consistently outperforms the plain baseline across most experimental configurations. This performance gap is particularly pronounced in the CNN models which are inherently better suited for capturing spatial features in image datasets than the MLP models. For instance, on the CIFAR10 (1), while the plain network performs at the level of random guessing, SAL achieves a mean accuracy of nearly 64\%. Performance gains are still observed in data-scarce regimes. Notably, when training the MLP on MNIST (1/100), SAL maintains high performance despite the drastic reduction in training data.

Additionally, we examined the gradient flow to evaluate the impact of SAL on backward propagation. Specifically, we computed the ratio of the weight gradient magnitude at each layer to that of the output layer. Considering that the early stages of training are heavily influenced by initialization and that gradients naturally diminish as the model converges in the later stages, we measured the gradient ratio at the midpoint of the total training duration determined by early stopping.
As shown in Figure \ref{fig:sal_gradient}, unlike plain networks, where gradient values diminish rapidly as they propagate away from the output layer, SAL maintains stable gradients through guidance from upper floors. This facilitates more effective backpropagation down to the earlier layers.

\begin{table}[t!]
    \caption{Performance comparison between SAL and plain training using sigmoid activation. Results are mean $\pm$ std over five independent runs. SAL demonstrates robust performance against the gradient vanishing problem induced by sigmoid activation.}
    \label{tab:Gradient_Vanishing}
    \centering
    \setlength{\tabcolsep}{3pt}
    \begin{tabular}{llcc}
    \toprule
    \textbf{Network} & \textbf{Dataset (Subset Ratio)} & \textbf{SAL (Acc)} & \textbf{plain (Acc)} \\
    \midrule
    MLP & MNIST (1) & \textbf{96.58 $\pm$ 0.10} & 96.10 $\pm$ 0.17 \\
    (Sigmoid) & MNIST (1/10) & \textbf{92.28 $\pm$ 0.56} & 88.56 $\pm$ 1.03 \\
     & MNIST (1/100) & \textbf{73.62 $\pm$ 5.14} & 16.93 $\pm$ 5.10 \\
    \cmidrule{2-4}
     & CIFAR10 (1) & \textbf{47.13 $\pm$ 1.67} & 41.16 $\pm$ 1.31 \\
     & CIFAR10 (1/10) & \textbf{27.89 $\pm$ 1.08} & 20.40 $\pm$ 1.41 \\
     & CIFAR10 (1/100) & \textbf{16.72 $\pm$ 1.44} & 13.73 $\pm$ 3.55 \\
    \cmidrule{2-4}
     & CIFAR100 (1) & \textbf{19.96 $\pm$ 0.64} & 1.00 $\pm$ 0.00 \\
     & CIFAR100 (1/10) & \textbf{6.74 $\pm$ 0.69} & 1.00 $\pm$ 0.00 \\
     & CIFAR100 (1/100) & \textbf{1.34 $\pm$ 0.38} & 1.00 $\pm$ 0.00 \\
    \cmidrule{2-4}
     & IMAGENETTE (1) & \textbf{37.11 $\pm$ 1.08} & 26.62 $\pm$ 9.80 \\
     & IMAGENETTE (1/10) & \textbf{21.40 $\pm$ 2.73} & 15.66 $\pm$ 0.59 \\
     & IMAGENETTE (1/100) & \textbf{13.27 $\pm$ 4.72} & 10.96 $\pm$ 2.18 \\
    \midrule
    CNN & MNIST (1) & \textbf{99.44 $\pm$ 0.05} & 11.35 $\pm$ 0.00 \\
    (Sigmoid) & MNIST (1/10) & \textbf{98.34 $\pm$ 0.07} & 11.35 $\pm$ 0.00 \\
     & MNIST (1/100) & \textbf{11.35 $\pm$ 0.00} & 10.92 $\pm$ 0.59 \\
    \cmidrule{2-4}
     & CIFAR10 (1) & \textbf{63.99 $\pm$ 0.88} & 10.00 $\pm$ 0.00 \\
     & CIFAR10 (1/10) & \textbf{40.68 $\pm$ 1.05} & 10.00 $\pm$ 0.00 \\
     & CIFAR10 (1/100) & \textbf{27.02 $\pm$ 2.29} & 10.00 $\pm$ 0.00 \\
    \cmidrule{2-4}
     & CIFAR100 (1) & \textbf{31.12 $\pm$ 1.18} & 1.00 $\pm$ 0.00 \\
     & CIFAR100 (1/10) & \textbf{9.93 $\pm$ 1.20} & 1.00 $\pm$ 0.00 \\
     & CIFAR100 (1/100) & \textbf{1.72 $\pm$ 0.74} & 1.00 $\pm$ 0.00 \\
    \cmidrule{2-4}
     & IMAGENETTE (1) & \textbf{64.33 $\pm$ 3.00} & 9.27 $\pm$ 0.38 \\
     & IMAGENETTE (1/10) & \textbf{28.28 $\pm$ 11.18} & 9.64 $\pm$ 0.50 \\
     & IMAGENETTE (1/100) & 9.46 $\pm$ 0.49 & \textbf{9.74 $\pm$ 0.66} \\
    \bottomrule
    \end{tabular}
\end{table}

These results underscore the objective of SAL in stabilizing gradients, which leads to enabling successful training of deeper networks. This suggests that the advantages of SAL are likely to be amplified in models with greater depth.

\setlength{\textfloatsep}{25pt}
\begin{figure}[!t]
    \centering
    \includegraphics[width=\columnwidth]{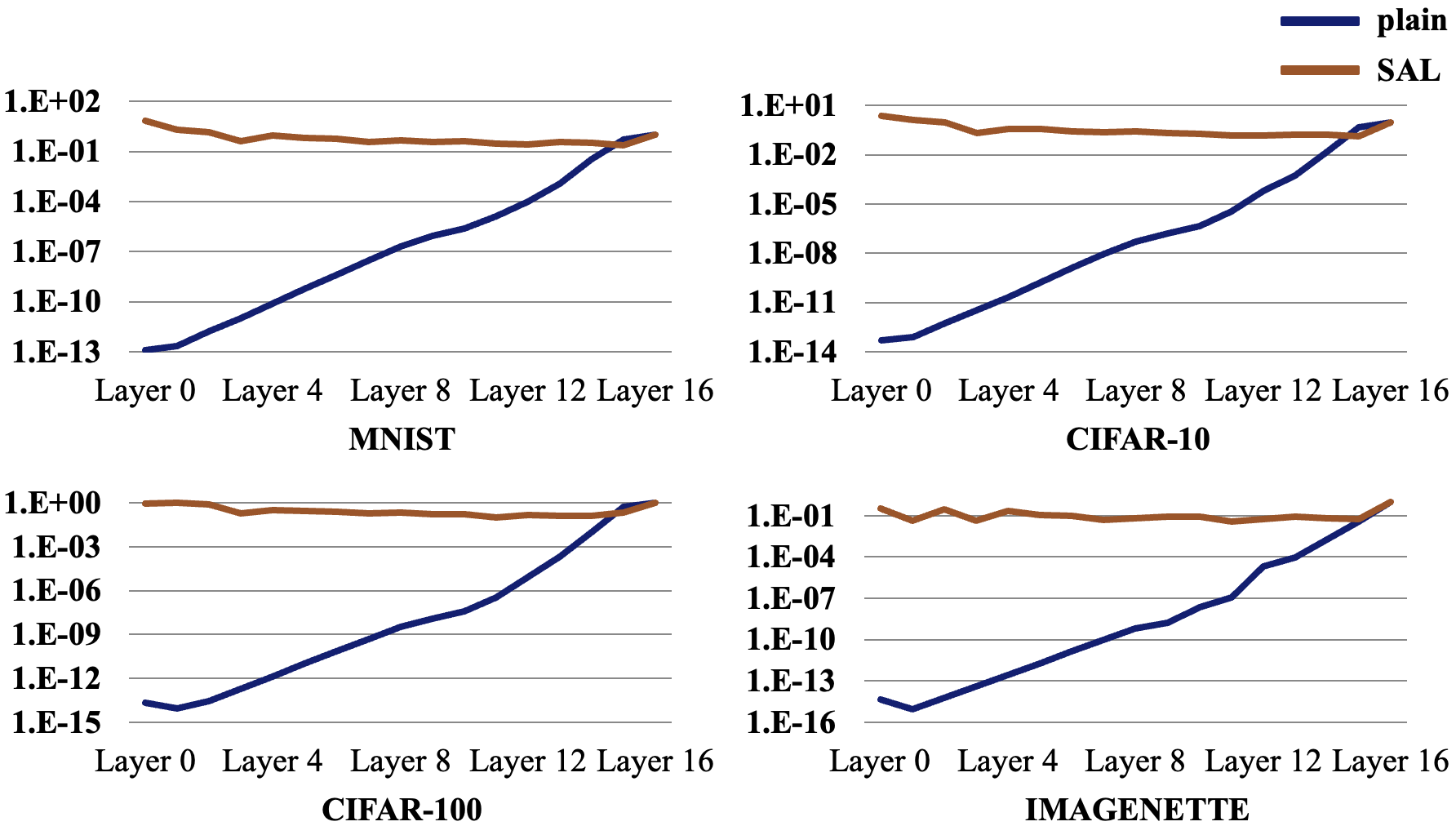}
    \caption{Log-scale gradient ratio across layers for various datasets between SAL and plain training. We conducted measurements on MLP models trained on image datasets. To clearly manifest the disparity in gradient flow, we extended the model depth to 16 layers.}
    \label{fig:sal_gradient}
\end{figure}

\subsection{General Performance Comparison}

\begin{table}[!t]
\caption{General performance comparison between SAL and plain training. SAL shows improvements in most datasets and limited data subsets.}
\label{tab:SAL_Performance}
\begin{center}
\setlength{\tabcolsep}{3pt}
\begin{tabular}{llcc}
\toprule
\textbf{Network} & \textbf{Dataset (Subset Ratio)} & \textbf{SAL (Acc)} & \textbf{plain (Acc)} \\
\midrule
MLP (ReLU) & MNIST (1) & 96.60 $\pm$ 0.37 & \textbf{97.34 $\pm$ 0.21} \\
& MNIST (1/10) & 92.55 $\pm$ 0.61 & \textbf{93.30 $\pm$ 0.45} \\
& MNIST (1/100) & \textbf{82.04 $\pm$ 0.44} & 76.59 $\pm$ 3.69 \\
\cmidrule{2-4}
& CIFAR10 (1) & 49.25 $\pm$ 1.20 & \textbf{50.51 $\pm$ 0.50} \\
& CIFAR10 (1/10) & \textbf{39.63 $\pm$ 1.58} & 28.57 $\pm$ 1.63 \\
& CIFAR10 (1/100) & \textbf{25.46 $\pm$ 1.16} & 19.95 $\pm$ 1.89 \\
\cmidrule{2-4}
& CIFAR100 (1) & \textbf{20.85 $\pm$ 0.13} & 17.20 $\pm$ 0.49 \\
& CIFAR100 (1/10) & \textbf{11.04 $\pm$ 0.79} & 4.42 $\pm$ 1.13 \\
& CIFAR100 (1/100) & \textbf{2.48 $\pm$ 0.93} & 1.59 $\pm$ 0.43 \\
\cmidrule{2-4}
& IMAGENETTE (1) & \textbf{41.29 $\pm$ 1.26} & 35.28 $\pm$ 2.64 \\
& IMAGENETTE (1/10) & \textbf{30.12 $\pm$ 1.91} & 23.57 $\pm$ 2.47 \\
& IMAGENETTE (1/100) & \textbf{13.58 $\pm$ 2.72} & 12.62 $\pm$ 3.98 \\
\midrule
CNN (ReLU) & MNIST (1) & \textbf{99.38 $\pm$ 0.10} & 99.20 $\pm$ 0.12 \\
& MNIST (1/10) & 97.42 $\pm$ 1.61 & \textbf{97.45 $\pm$ 0.45} \\
& MNIST (1/100) & 73.67 $\pm$ 34.86 & \textbf{85.02 $\pm$ 3.49} \\
\cmidrule{2-4}
& CIFAR10 (1) & \textbf{76.17 $\pm$ 1.42} & 74.64 $\pm$ 0.80 \\
& CIFAR10 (1/10) & \textbf{51.55 $\pm$ 1.80} & 46.87 $\pm$ 1.73 \\
& CIFAR10 (1/100) & \textbf{31.21 $\pm$ 2.39} & 26.82 $\pm$ 1.37 \\
\cmidrule{2-4}
& CIFAR100 (1) & \textbf{42.52 $\pm$ 1.36} & 31.58 $\pm$ 1.48 \\
& CIFAR100 (1/10) & \textbf{12.62 $\pm$ 0.94} & 8.76 $\pm$ 0.98 \\
& CIFAR100 (1/100) & \textbf{1.35 $\pm$ 0.29} & 1.28 $\pm$ 0.35 \\
\cmidrule{2-4}
& IMAGENETTE (1) & \textbf{76.34 $\pm$ 1.06} & 69.67 $\pm$ 3.25 \\
& IMAGENETTE (1/10) & \textbf{44.80 $\pm$ 3.51} & 32.22 $\pm$ 3.65 \\
& IMAGENETTE (1/100) & 13.20 $\pm$ 4.27 & \textbf{13.65 $\pm$ 1.78} \\
\midrule
LSTM & CARER (1) & \textbf{90.83 $\pm$ 0.50} & 84.77 $\pm$ 5.44 \\
& CARER (1/10) & 32.50 $\pm$ 2.35 & \textbf{35.60 $\pm$ 1.90} \\
& CARER (1/100) & \textbf{34.81 $\pm$ 0.63} & 33.61 $\pm$ 2.55 \\
\cmidrule{2-4}
& TREC (1) & \textbf{77.80 $\pm$ 1.07} & 65.76 $\pm$ 5.98 \\
& TREC (1/10) & \textbf{50.16 $\pm$ 4.54} & 46.92 $\pm$ 3.32 \\
& TREC (1/100) & \textbf{31.64 $\pm$ 6.57} & 27.20 $\pm$ 5.99 \\
\cmidrule{2-4}
& TREC\_FINE (1) & \textbf{55.56 $\pm$ 2.79} & 30.92 $\pm$ 16.35 \\
& TREC\_FINE (1/10) & \textbf{31.24 $\pm$ 8.13} & 17.96 $\pm$ 9.53 \\
& TREC\_FINE (1/100) & \textbf{22.24 $\pm$ 9.47} & 11.00 $\pm$ 0.00 \\
\bottomrule
\end{tabular}
\end{center}
\end{table}
\begin{figure*}[!t]
    \centering
    \includegraphics[width=\textwidth]{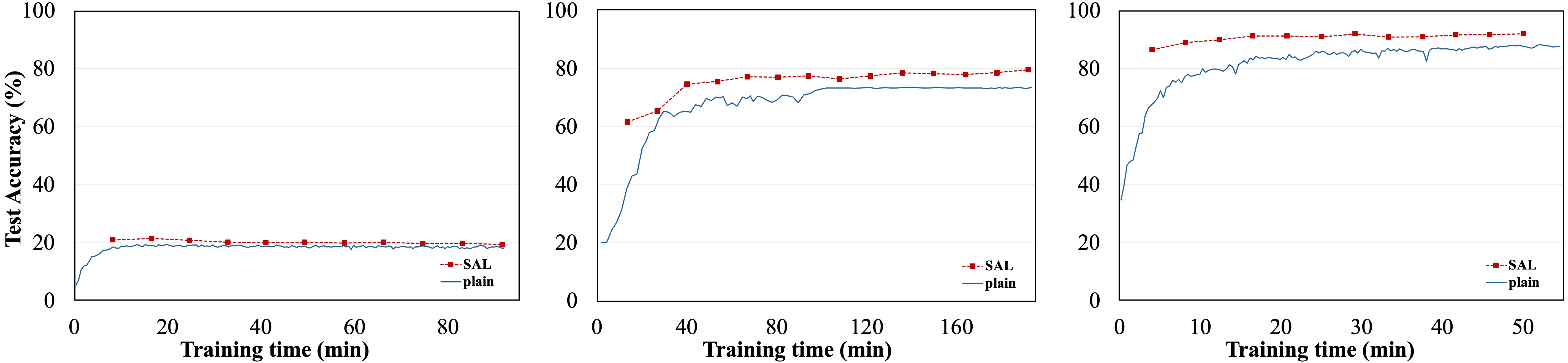}
    \caption{Test accuracy over wall-clock training time for SAL and plain training
    under a time-matched setting. SAL is evaluated at each step and plain at each
    epoch. SAL consistently achieves higher accuracy throughout the training time
    across MLP (CIFAR100), CNN (Imagenette), and LSTM (CARER) configurations
    (from left to right).}
    \label{fig:time_matched}
\end{figure*}

We now evaluate SAL performance under general training configurations employing the ReLU activation function on the MLP and the CNN models. For the LSTM, we followed the standard model utilizing sigmoid and tanh activation functions. Table \ref{tab:SAL_Performance} presents the comparative results between SAL and plain training. SAL generally shows better performance across the majority of datasets even in subset datasets.  
For instance, using the LSTM model on TREC (1) yields superior performance with low variance, with similarly strong performance observed on using the MLP model on CIFAR100 (1/10) and the CNN model on IMAGENETTE (1/10).
Although marginal performance drops were observed in a small number of cases,
these can be attributed to two factors: task simplicity, where plain training
already converges reliably and SAL's guidance becomes redundant, and
limitations of the current element-wise guidance mechanism. The latter is
expected to be addressed by more sophisticated guidance strategies.
Overall, these results demonstrate that SAL consistently achieves
competitive or superior performance compared to plain training across
diverse architectures and datasets under standard training configurations. Moreover, SAL performance remains superior when using smaller subsets of datasets, demonstrating the method's robustness and generalization ability. The consistent performance advantage of SAL in data-scarce subsets (1/10, 1/100) across most datasets also provides indirect empirical support for the overfitting mitigation effect established theoretically in Section~\ref{sec:theoretical}. By leveraging its hierarchical structure, SAL progressively 
refines the coarse-grained representations, 
which appears to facilitate more effective optimization 
compared to conventional training.

\begin{table}[!t]
\caption{Peak VRAM and wall-clock training time comparison between SAL and plain training (Tesla V100).}
\label{tab:overhead}
\begin{center}
\setlength{\tabcolsep}{4pt}
\resizebox{\columnwidth}{!}{\begin{tabular}{llcccc}
\toprule
\textbf{Network} & \textbf{Dataset}
  & \textbf{Plain VRAM} & \textbf{SAL VRAM}
  & \textbf{Plain (min)} & \textbf{SAL (min)} \\
\midrule
MLP    & MNIST      & 30.8 MB  & 31.9 MB  & 17  & 123 \\
              & CIFAR10    & 42.1 MB  & 43.4 MB  & 10  & 91  \\
              & CIFAR100   & 84.1 MB  & 89.5 MB  & 11  & 91  \\
              & IMAGENETTE & 821.2 MB & 898.7 MB & 9   & 94  \\
\midrule
CNN    & MNIST      & 335.8 MB & 401.6 MB & 8   & 152 \\
              & CIFAR10    & 336.3 MB & 401.1 MB & 9   & 103 \\
              & CIFAR100   & 336.5 MB & 401.4 MB & 10  & 115 \\
              & IMAGENETTE & 3.5 GB   & 5.0 GB   & 27  & 191 \\
\midrule
LSTM          & CARER    & 1.1 GB   & 1.3 GB   & 11  & 54  \\
              & TREC       & 584.6 MB & 660.7 MB & 2   & 21  \\
              & TREC\_FINE & 584.6 MB & 660.8 MB & 2   & 24  \\
\bottomrule
\end{tabular}}
\end{center}
\end{table}

\subsection{Computational Overhead and Training Efficiency}
\label{sec:overhead}

SAL trains multiple floor models sequentially, which raises the question of whether its performance gains come at an acceptable computational cost. We analyze this from two complementary perspectives, peak VRAM usage and wall-clock training time.

\textbf{VRAM and Training Time.}
We measured peak GPU memory allocation and total wall-clock training time
on Tesla V100 for both SAL and plain baseline across all architectures and
datasets, with early stopping determined by a patience of 10 epochs for
plain and 10 steps for SAL, as described in Section~\ref{sec:experiments setup}. As shown in
Table~\ref{tab:overhead}, SAL incurs only a modest VRAM overhead relative to plain
training. Across MLP configurations, SAL requires roughly 3--10\% more memory.
CNN and LSTM configurations typically show increases of approximately 13--24\%,
with the exception of CNN on IMAGENETTE where the overhead reaches approximately
43\% due to the larger spatial feature maps at 160$\times$160 resolution. This
overhead stems from the multi-layer MSE guidance, which requires storing intermediate
activations of both upper and lower floors simultaneously, and remains confined to a
single guidance phase rather than accumulating over the full training duration.
Regarding training time, SAL requires more total wall-clock time than plain training
due to its multi-floor pipeline. However, the time-matched comparison in
Figure~\ref{fig:time_matched} demonstrates that SAL consistently achieves higher
accuracy than plain training as detailed below.

\textbf{Time-Matched Comparison.}
Since 1 step of SAL involves more total epochs than a single plain training run,
a direct step-to-epoch comparison would be unfair for evaluating training efficiency.
We therefore evaluate SAL and plain training under a time-matched setting. Both
methods are run for the same wall-clock time, determined by the point at which
SAL triggers early stopping, and we report test accuracy against wall-clock training
time. Figure~\ref{fig:time_matched} presents results for three representative
configurations — MLP on CIFAR100, CNN on Imagenette, and LSTM on CARER — covering
all three architecture types evaluated in this work. SAL generally achieves higher
test accuracy than plain training throughout the entire training duration. Despite
extended training beyond the early stopping point, the plain baseline remained below SAL across all three cases. Notably, in the case of LSTM on CARER, SAL reaches its performance ceiling early in training, while plain training continues to improve
slowly and ultimately plateaus at a lower accuracy level. This suggests that hierarchical guidance of SAL leads to more efficient use of compute.

\subsection{Hyperparameter Sensitivity Analysis}
\begin{table}[!t]
\caption{Hyperparameter sensitivity analysis of training epochs ($t$) and guidance epochs ($r$) across all datasets. The hyperparameter settings used in the main experiments are $t=5$, $r=10$ for MLP and CNN, and $t=2$, $r=10$ for LSTM. The best result in each row is shown in bold.}
\label{tab:hyperparam_epoch}
\begin{center}
\setlength{\tabcolsep}{3pt}
\resizebox{\columnwidth}{!}{\begin{tabular}{llccccc}
\toprule
\textbf{Network} & \textbf{Dataset} & $t{=}2, r{=}5$ & $t{=}5, r{=}5$ & $t{=}2, r{=}10$ & $t{=}5, r{=}10$ & plain \\
\midrule
MLP & MNIST (1)          & 96.98$\pm$0.21 & 96.88$\pm$0.32 & 96.80$\pm$0.19 & 96.60$\pm$0.37 & \textbf{97.34$\pm$0.21} \\
    & MNIST (1/10)       & 92.83$\pm$0.39 & 92.55$\pm$0.51 & 92.50$\pm$0.14 & 92.55$\pm$0.61 & \textbf{93.30$\pm$0.45} \\
    & MNIST (1/100)      & 83.14$\pm$0.70 & 81.40$\pm$0.96 & \textbf{83.87$\pm$0.89} & 82.04$\pm$0.44 & 76.59$\pm$3.69 \\
    & CIFAR10 (1)        & 50.12$\pm$0.61 & 50.11$\pm$0.58 & 49.72$\pm$0.88 & 49.25$\pm$1.20 & \textbf{50.51$\pm$0.50} \\
    & CIFAR10 (1/10)     & 38.01$\pm$0.91 & 39.35$\pm$1.06 & 38.28$\pm$0.73 & \textbf{39.63$\pm$1.58} & 28.57$\pm$1.63 \\
    & CIFAR10 (1/100)    & 23.04$\pm$1.70 & 22.80$\pm$2.33 & 22.63$\pm$0.84 & \textbf{25.46$\pm$1.16} & 19.95$\pm$1.89 \\
    & CIFAR100 (1)       & \textbf{21.50$\pm$0.46} & 20.27$\pm$0.34 & 20.81$\pm$0.57 & 20.85$\pm$0.13 & 17.20$\pm$0.49 \\
    & CIFAR100 (1/10)    & 9.32$\pm$0.84 & \textbf{11.22$\pm$0.53} & 9.72$\pm$0.48 & 11.04$\pm$0.79 & 4.42$\pm$1.13 \\
    & CIFAR100 (1/100)   & 2.15$\pm$0.69 & 1.85$\pm$0.76 & \textbf{2.52$\pm$0.45} & 2.48$\pm$0.93 & 1.59$\pm$0.43 \\
    & IMAGENETTE (1)     & 40.25$\pm$0.78 & \textbf{42.47$\pm$0.41} & 41.44$\pm$0.63 & 41.29$\pm$1.26 & 35.28$\pm$2.64 \\
    & IMAGENETTE (1/10)  & 26.10$\pm$3.15 & 27.50$\pm$2.32 & 25.76$\pm$2.18 & \textbf{30.12$\pm$1.91} & 23.57$\pm$2.47 \\
    & IMAGENETTE (1/100) & 10.54$\pm$1.85 & 11.80$\pm$2.05 & 13.28$\pm$1.74 & \textbf{13.58$\pm$2.72} & 12.62$\pm$3.98 \\
\midrule
CNN & MNIST (1)          & 99.35$\pm$0.12 & 99.23$\pm$0.07 & 99.32$\pm$0.09 & \textbf{99.38$\pm$0.10} & 99.20$\pm$0.12 \\
    & MNIST (1/10)       & 97.74$\pm$0.17 & \textbf{97.98$\pm$0.12} & 97.78$\pm$0.05 & 97.42$\pm$1.61 & 97.45$\pm$0.45 \\
    & MNIST (1/100)      & 17.78$\pm$13.76 & 67.57$\pm$19.05 & 24.02$\pm$28.32 & 73.67$\pm$34.86 & \textbf{85.02$\pm$3.49} \\
    & CIFAR10 (1)        & \textbf{77.51$\pm$0.89} & 75.86$\pm$2.03 & 76.23$\pm$1.27 & 76.17$\pm$1.42 & 74.64$\pm$0.80 \\
    & CIFAR10 (1/10)     & \textbf{54.15$\pm$1.45} & 52.52$\pm$1.70 & 53.12$\pm$1.14 & 51.55$\pm$1.80 & 46.87$\pm$1.73 \\
    & CIFAR10 (1/100)    & 27.93$\pm$10.17 & 24.10$\pm$8.71 & \textbf{32.08$\pm$2.91} & 31.21$\pm$2.39 & 26.82$\pm$1.37 \\
    & CIFAR100 (1)       & 41.44$\pm$1.34 & \textbf{42.86$\pm$2.25} & 41.59$\pm$0.81 & 42.52$\pm$1.36 & 31.58$\pm$1.48 \\
    & CIFAR100 (1/10)    & \textbf{12.80$\pm$0.98} & 11.71$\pm$1.26 & 11.95$\pm$0.84 & 12.62$\pm$0.94 & 8.76$\pm$0.98 \\
    & CIFAR100 (1/100)   & 1.58$\pm$0.43 & 1.19$\pm$0.33 & \textbf{1.82$\pm$0.92} & 1.35$\pm$0.29 & 1.28$\pm$0.35 \\
    & IMAGENETTE (1)     & \textbf{76.49$\pm$1.63} & 75.06$\pm$3.06 & 76.24$\pm$1.95 & 76.34$\pm$1.06 & 69.67$\pm$3.25 \\
    & IMAGENETTE (1/10)  & 42.77$\pm$1.27 & 44.06$\pm$4.83 & 44.01$\pm$3.61 & \textbf{44.80$\pm$3.51} & 32.22$\pm$3.65 \\
    & IMAGENETTE (1/100) & 12.05$\pm$3.87 & 13.37$\pm$4.60 & 11.51$\pm$4.10 & 13.20$\pm$4.27 & \textbf{13.65$\pm$1.78} \\
\midrule
LSTM & CARER (1)          & 90.78$\pm$0.50 & 89.55$\pm$0.99 & \textbf{90.83$\pm$0.50} & 88.57$\pm$2.00 & 84.77$\pm$5.44 \\
     & CARER (1/10)       & 34.00$\pm$1.68 & 33.38$\pm$3.74 & 32.50$\pm$2.35 & \textbf{36.95$\pm$4.84} & 35.60$\pm$1.90 \\
     & CARER (1/100)      & \textbf{35.03$\pm$0.33} & 34.77$\pm$0.18 & 34.81$\pm$0.63 & 34.80$\pm$0.09 & 33.61$\pm$2.55 \\
     & TREC (1)             & 76.88$\pm$3.26 & 73.72$\pm$4.18 & \textbf{77.80$\pm$1.07} & 76.40$\pm$1.62 & 65.76$\pm$5.98 \\
     & TREC (1/10)          & \textbf{51.76$\pm$5.54} & 47.96$\pm$6.68 & 50.16$\pm$4.54 & 45.40$\pm$4.24 & 46.92$\pm$3.32 \\
     & TREC (1/100)         & 29.04$\pm$6.66 & 25.32$\pm$9.41 & \textbf{31.64$\pm$6.57} & 28.68$\pm$11.44 & 27.20$\pm$5.99 \\
     & TREC\_FINE (1)      & 54.36$\pm$2.18 & 49.88$\pm$3.05 & \textbf{55.56$\pm$2.79} & 48.64$\pm$3.47 & 30.92$\pm$16.35 \\
     & TREC\_FINE (1/10)   & \textbf{33.96$\pm$5.21} & 28.12$\pm$6.93 & 31.24$\pm$8.13 & 31.08$\pm$4.50 & 17.96$\pm$9.53 \\
     & TREC\_FINE (1/100)  & 14.96$\pm$10.77 & 20.08$\pm$12.52 & \textbf{22.24$\pm$9.47} & 20.04$\pm$10.43 & 11.00$\pm$0.00 \\
\bottomrule
\end{tabular}}
\end{center}
\end{table}

We investigate the sensitivity of SAL to three key hyperparameters: the
number of direct training epochs ($t$), guidance epochs ($r$), and floor
size ($F$). All results are reported as mean $\pm$ standard deviation over
five independent runs.
Table~\ref{tab:hyperparam_epoch} reports accuracy across four $(t, r)$ settings---$(2,5)$, $(5,5)$, $(2,10)$, and $(5,10)$---for all datasets. Overall, performance varies
slightly across the range of $t$ and $r$ values, suggesting that SAL
is not highly sensitive to these hyperparameters. For MLP and CNN, no
single $(t, r)$ setting consistently dominates across all datasets. For
LSTM, $(t{=}2, r{=}10)$ tends to perform favorably, which motivates its
use as the main experimental setting for text datasets. Across most $(t, r)$ configurations and datasets, SAL performs better than
plain training, with exceptions on simpler tasks such as MLP on
MNIST and CIFAR10, where plain training already converges reliably.
Table~\ref{tab:hyperparam_floor} reports accuracy across floor sizes $F \in
\{2, 3, 4\}$, with $F{=}3$ used in the main experiments. The models at each floor are constructed in the same manner as described in Section~\ref{sec:experiments setup}. Performance across floor sizes is generally comparable, with no single value of $F$ consistently outperforming the others across all settings.
Nevertheless, $F{=}3$ tends to yield competitive performance across a broader
range of datasets and data-scarce regimes compared to $F{=}2$ and $F{=}4$,
motivating its selection as the default setting. Regardless of the choice of $F$, SAL
consistently outperforms plain training across the majority of datasets and
subset ratios, confirming that the performance gains of SAL are not
attributable to a specific floor size but stem from the hierarchical
guidance mechanism itself.

\begin{table}[!t]
\caption{Hyperparameter sensitivity analysis of floor size ($F$) across all datasets. The hyperparameter setting used in the main experiments is $F=3$ for all models. The best result in each row is shown in bold.}
\label{tab:hyperparam_floor}
\begin{center}
\setlength{\tabcolsep}{2pt}
\resizebox{\columnwidth}{!}{\begin{tabular}{llcccc}
\toprule
\textbf{Network} & \textbf{Dataset} & $F{=}2$ & $F{=}4$ & $F{=}3$ & plain \\
\midrule
MLP & MNIST (1)          & 97.25$\pm$0.39 & 96.46$\pm$0.12 & 96.60$\pm$0.37 & \textbf{97.34$\pm$0.21} \\
           & MNIST (1/10)       & 92.67$\pm$0.49 & 91.75$\pm$0.55 & 92.55$\pm$0.61 & \textbf{93.30$\pm$0.45} \\
           & MNIST (1/100)      & 77.92$\pm$3.41 & \textbf{84.47$\pm$0.43} & 82.04$\pm$0.44 & 76.59$\pm$3.69 \\
           & CIFAR10 (1)        & \textbf{50.57$\pm$0.63} & 49.22$\pm$0.33 & 49.25$\pm$1.20 & 50.51$\pm$0.50 \\
           & CIFAR10 (1/10)     & 38.67$\pm$1.46 & 38.84$\pm$1.07 & \textbf{39.63$\pm$1.58} & 28.57$\pm$1.63 \\
           & CIFAR10 (1/100)    & 24.41$\pm$1.47 & \textbf{25.72$\pm$1.13} & 25.46$\pm$1.16 & 19.95$\pm$1.89 \\
           & CIFAR100 (1)       & 20.41$\pm$0.49 & \ 20.85$\pm$0.46 & \textbf{20.85$\pm$0.13} & 17.20$\pm$0.49 \\
           & CIFAR100 (1/10)    & 10.21$\pm$0.71 & \textbf{11.47$\pm$0.37} & 11.04$\pm$0.79 & 4.42$\pm$1.13 \\
           & CIFAR100 (1/100)   & \textbf{3.22$\pm$0.66}  & 2.39$\pm$0.46  & 2.48$\pm$0.93  & 1.59$\pm$0.43 \\
           & IMAGENETTE (1)     & \textbf{43.47$\pm$0.96} & 40.71$\pm$0.57 & 41.29$\pm$1.26 & 35.28$\pm$2.64 \\
           & IMAGENETTE (1/10)  & \textbf{32.93$\pm$1.29} & 24.58$\pm$1.62 & 30.12$\pm$1.91 & 23.57$\pm$2.47 \\
           & IMAGENETTE (1/100) & \textbf{13.89$\pm$2.04} & 13.57$\pm$2.50 & 13.58$\pm$2.72 & 12.62$\pm$3.98 \\
\midrule
CNN & MNIST (1)          & 99.21$\pm$0.12 & 99.17$\pm$0.15 & \textbf{99.38$\pm$0.10} & 99.20$\pm$0.12 \\
           & MNIST (1/10)       & \textbf{98.14$\pm$0.30} & 97.19$\pm$1.28 & 97.42$\pm$1.61 & 97.45$\pm$0.45 \\
           & MNIST (1/100)      & 84.85$\pm$13.17& 64.92$\pm$12.23& 73.67$\pm$34.86& \textbf{85.02$\pm$3.49} \\
           & CIFAR10 (1)        & \textbf{76.85$\pm$0.58} & 76.85$\pm$1.64 & 76.17$\pm$1.42 & 74.64$\pm$0.80 \\
           & CIFAR10 (1/10)     & \textbf{54.25$\pm$2.87} & 49.38$\pm$2.69 & 51.55$\pm$1.80 & 46.87$\pm$1.73 \\
           & CIFAR10 (1/100)    & 29.42$\pm$2.38 & 29.67$\pm$0.64 & \textbf{31.21$\pm$2.39} & 26.82$\pm$1.37 \\
           & CIFAR100 (1)       & 42.29$\pm$1.71 & 41.70$\pm$2.08 & \textbf{42.52$\pm$1.36} & 31.58$\pm$1.48 \\
           & CIFAR100 (1/10)    & \textbf{12.68$\pm$1.99} & 9.24$\pm$4.97  & 12.62$\pm$0.94 & 8.76$\pm$0.98 \\
           & CIFAR100 (1/100)   & 1.67$\pm$0.35  & \textbf{1.72$\pm$0.31}  & 1.35$\pm$0.29  & 1.28$\pm$0.35 \\
           & IMAGENETTE (1)     & \textbf{76.86$\pm$1.40} & 76.20$\pm$1.64 & 76.34$\pm$1.06 & 69.67$\pm$3.25 \\
           & IMAGENETTE (1/10)  & 39.17$\pm$3.30 & 44.20$\pm$4.16 & \textbf{44.80$\pm$3.51} & 32.22$\pm$3.65 \\
           & IMAGENETTE (1/100) & 9.94$\pm$0.57  & 12.91$\pm$3.47 & 13.20$\pm$4.27 & \textbf{13.65$\pm$1.78} \\
\midrule
LSTM & CARER (1)       & 90.81$\pm$0.81 & \textbf{91.06$\pm$0.98} & 90.83$\pm$0.50 & 84.77$\pm$5.44 \\
     & CARER (1/10)    & \textbf{36.51$\pm$2.81} & 33.23$\pm$3.46 & 32.50$\pm$2.35 & 35.60$\pm$1.90 \\
     & CARER (1/100)   & 34.80$\pm$0.11 & 34.42$\pm$0.45 & \textbf{34.81$\pm$0.63} & 33.61$\pm$2.55 \\
     & TREC (1)          & \textbf{78.60$\pm$1.04} & 76.56$\pm$2.01 & 77.80$\pm$1.07 & 65.76$\pm$5.98 \\
     & TREC (1/10)       & 52.00$\pm$11.87& \textbf{56.16$\pm$5.66} & 50.16$\pm$4.54 & 46.92$\pm$3.32 \\
     & TREC (1/100)      & 30.08$\pm$5.07 & \textbf{38.96$\pm$5.05} & 31.64$\pm$6.57 & 27.20$\pm$5.99 \\
     & TREC\_FINE (1)    & 51.92$\pm$3.65 & 54.48$\pm$4.29 & \textbf{55.56$\pm$2.79} & 30.92$\pm$16.35 \\
     & TREC\_FINE (1/10) & 35.00$\pm$11.19& \textbf{37.20$\pm$2.32} & 31.24$\pm$8.13 & 17.96$\pm$9.53 \\
     & TREC\_FINE (1/100)& 11.88$\pm$1.97 & 21.44$\pm$7.72 & \textbf{22.24$\pm$9.47} & 11.00$\pm$0.00 \\
\bottomrule
\end{tabular}}
\end{center}
\end{table}

\subsection{Practical Application on Data-Scarce Medical Imaging}
\label{sec:busi}
For practical application testing, we further validated our method on the Breast Ultrasound Images (BUSI) dataset \cite{al2020dataset}. This dataset poses challenges due to limited data availability and the inherent risk of overfitting, along with the need for fine-grained diagnosis. The dataset comprises 780 images from 600 patients, categorized into benign (437), malignant (210), and normal (133) classes. The MLP and CNN models were configured with a depth of 8 layers and 512 hidden units for both the plain model and the bottom model of SAL. Subsequently, the dataset was partitioned into training, validation, and test sets with a ratio of 70\%, 15\%, and 15\%, respectively, while maintaining the original class distribution. During the training phase, early stopping was applied with a patience of 10 epochs for plain network and 10 steps for SAL based on the validation AUC. In contrast to the main experiments, Batch Normalization and Dropout were applied to both the plain and SAL networks to maximize their performance. Figure \ref{fig:BUSI_Results} shows that SAL achieves superior mean performance across all metrics-AUROC, F1-Score, Sensitivity, Specificity, Accuracy, marking average AUC improvements of 17.61\% for MLP and 16.46\% for CNN over the plain network. These results suggest that SAL is generally more effective in mitigating overfitting compared to the plain baseline, even under fine-grained classification scenarios.

Collectively, these empirical findings corroborate the 
effectiveness of SAL in reducing overfitting while maintaining 
competitive performance across diverse architectures, datasets, 
and data regimes.

\begin{figure*}[!t]
    \centering
    \subfloat[]{\includegraphics[width=0.48\textwidth]{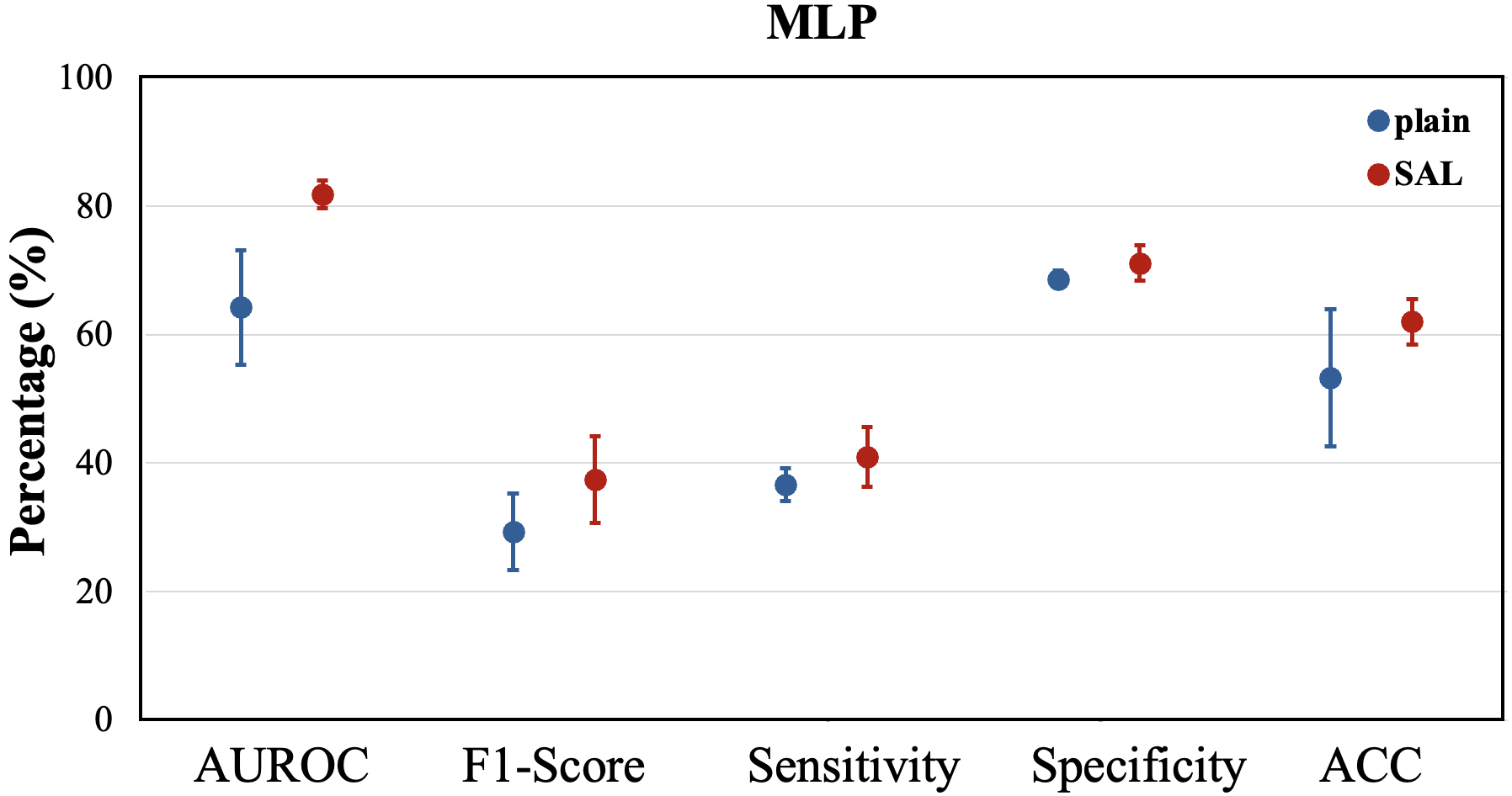}}
    \hfill
    \subfloat[]{\includegraphics[width=0.48\textwidth]{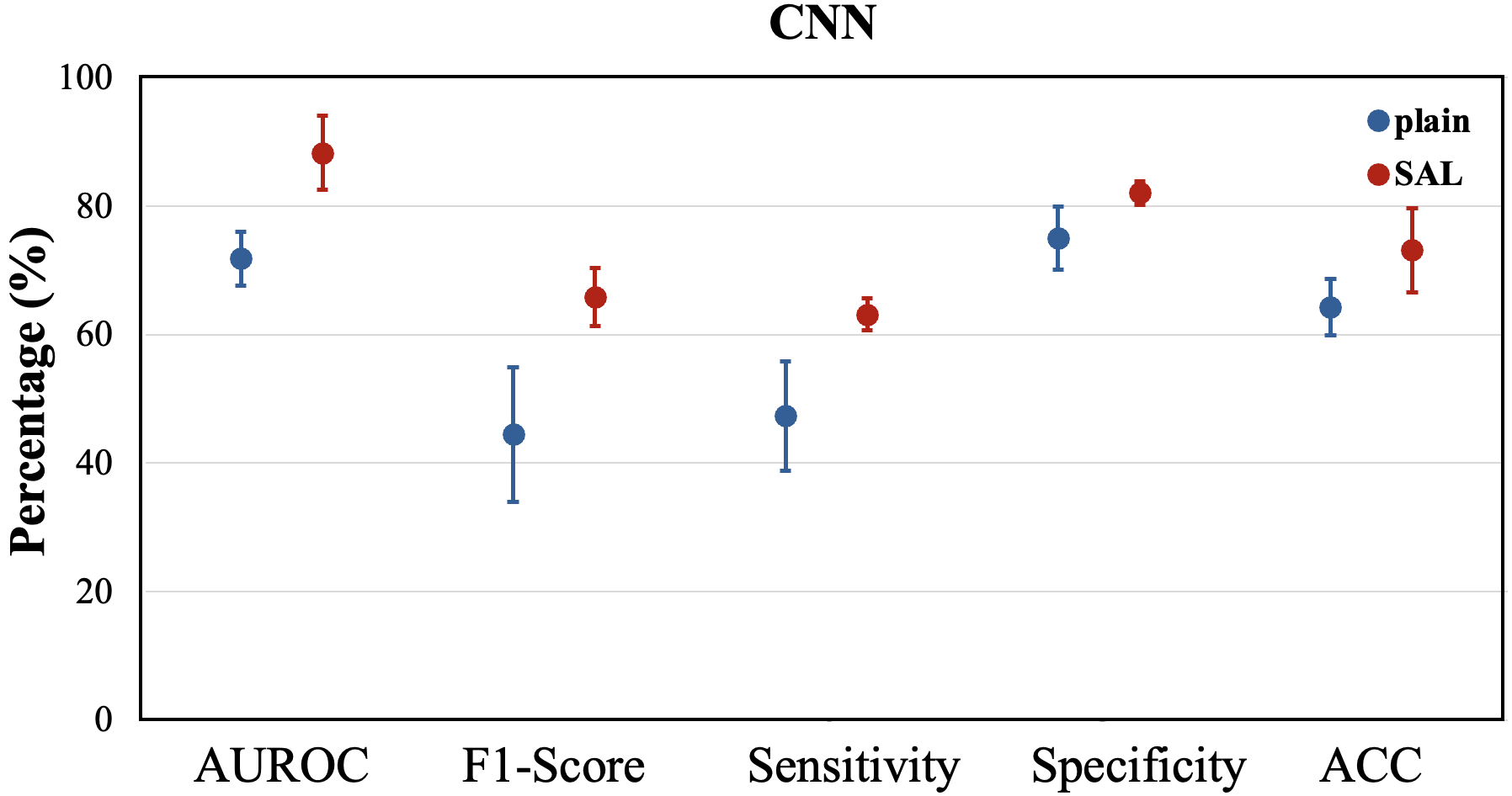}}
    \caption{Performance comparison of SAL versus plain baseline on the BUSI dataset using MLP (left) and CNN (right) in terms of AUROC, F1-Score, Sensitivity, Specificity, and ACC. Points and error bars denote the mean $\pm$ standard deviation from five independent runs.}
    \label{fig:BUSI_Results}
\end{figure*}

\section{Conclusion and Future Work}
In this paper, we proposed Self-Abstraction Learning, a hierarchical approach to improve the efficacy and stability of training deep neural networks. By refining models progressively from simple to complex structures, SAL tackles critical challenges like gradient vanishing and overfitting while optimizing performance. Our experiments across various architectures and datasets show that SAL exceeds conventional training method, delivering effective optimization, reduced overfitting, and improved generalization—especially in data-scarce settings. In conclusion, our approach achieves a more favorable bias-variance trade-off, successfully mitigating the high variance of deep models while preserving their low bias. These findings suggest the potential of SAL as a stable training 
solution. Nevertheless, SAL incurs additional training overhead compared to 
plain baselines, as the hierarchical multi-floor structure requires recursively 
training multiple models per step.
 
Future research on SAL can explore several promising directions. One key area is applying SAL to advanced architectures like Transformers \cite{vaswani2017attention} to broaden its use in tasks such as natural language processing and image recognition \cite{dosovitskiy2020image}. Furthermore, refining the guidance mechanism represents a critical avenue for improvement. While the current SAL implementation relies on direct element-wise feature matching, this approach may impose overly rigid constraints on the lower floor model. Future work could adopt more sophisticated feature transfer techniques, such as relation-based \cite{park2019relational} or attention-based \cite{zagoruyko2016paying} mechanisms, to enable flexible structural adaptation without strict value alignment.
Integrating Knowledge Distillation \cite{hinton2015distilling} into the SAL framework could effectively boost model performance, offering a synergy between our `simple-to-complex' Knowledge Guidance and `complex-to-simple' Knowledge Distillation. Investigating residual SAL variants, where the target is defined as the residual difference between the ground truth and the upper floor's output, may improve efficiency for real-time applications. This approach effectively exploits the speed-accuracy trade-off, which is particularly critical in industrial automation \cite{wang2018deep}. These extensions aim to further strengthen the potential of SAL in deep learning and reduce the total wall-clock training time required for convergence.

\appendices
\section{\\ Proof of Theorem 4.1}
\label{app:proof_of_theorem_4.1}

\begin{lemma}
\label{app:lemma}
For an input logit vector $\mathbf{z} \in \mathbb{R}^n$ and the corresponding target vector $\mathbf{y} \in [0,1]^n$, the CE loss function combined with the softmax function is Lipschitz continuous. 
\end{lemma}

\textit{Proof.}
Let the activation be $\mathbf{a} = s(\mathbf{z})$. The loss function is defined as follows:
\begin{equation*}
L(\mathbf{z}) = -\sum_{i=1}^n y_i\log a_i.
\end{equation*}
Then, the partial derivative with respect to each component $z_k$ of $\mathbf{z}$ is
\begin{equation*}
\frac{\partial L}{\partial z_k} = -y_k + \sum_{i=1}^n y_i a_k \leq -y_k + n a_k.
\end{equation*}
We now find the upper bound for the squared $\ell_2$ norm of this gradient. Since each component $a_i$ and $y_i$ is bounded in $[0,1]$, the following inequality holds.
\begin{equation*}
\lVert\nabla L(\mathbf{z})\rVert_2^2 \leq \sum_{k=1}^n (n a_k - y_k)^2 \leq n^3.
\end{equation*}
Hence,
\begin{equation*}
\lVert\nabla L(\mathbf{z})\rVert_2 \le n\sqrt{n}.
\end{equation*}
Applying the mean value theorem and the Cauchy-Schwarz inequality, for any $\mathbf{z_1} \in \mathbb{R}^n$ and $\mathbf{z_2} \in \mathbb{R}^n$, there exists $\mathbf{c} \in \mathbb{R}^n$ such that
\begin{equation*}
|L(\mathbf{z_1}) - L(\mathbf{z_2})| = |\nabla L(\mathbf{c})^T (\mathbf{z_1} - \mathbf{z_2})| \le n\sqrt{n} \, \lVert\mathbf{z_1} - \mathbf{z_2}\rVert_2.
\end{equation*}
Thus, the loss function $L$ is Lipschitz continuous.
\hfill $\Box$ 

Consider the case where $\mathbf{y}$ is a one-hot vector corresponding to the true class $j$. Based on the derivative derived in Lemma \ref{app:lemma}, the gradient is simplified as follows:
\begin{equation*}
\frac{\partial L}{\partial z_k} = a_k - y_k.
\end{equation*}
Hence, the squared $\ell_2$ norm of the gradient is
\begin{align}
\lVert\nabla L(\mathbf{z})\rVert_2^2
&= \sum_{k=1}^n(a_k-y_k)^2 \\
&= (a_j-1)^2 + \sum_{k\ne j}(a_k-0)^2 \\
&= 1 - 2a_j + \sum_{k=1}^n a_k^2.
\end{align}
Since $\mathbf{a}$ is a probability vector from the softmax function, we obtain the following inequality.
\begin{align}
\lVert\nabla L(\mathbf{z})\rVert_2^2
&\leq 2-2a_j \\
&\le 2.
\end{align}
Consequently, when the target $\mathbf{y}$ is a one-hot vector, the Lipschitz constant is tightened to $\sqrt{2}$.
{
\renewcommand{\thetheorem}{4.1}
\begin{theorem}[\textbf{Generalization Error Bound}]
Let $f \in \mathcal{F}$ and $h \in \mathcal{H}$ be functions satisfying $\mathbb{E}[\|f(x) - h(x)\|_2] < \epsilon$ for a small $\epsilon > 0$.
There exists a constant $C > 0$ such that
\begin{equation*}
R(f) \le R(h) + C \cdot \epsilon.
\end{equation*}
\end{theorem}
}
\textit{Proof.} 
For each model's output $f(x)$ and $h(x)$ about the input data $x$ and the CE loss function $l$,
\begin{equation*}
l(f(x), y) \le l(h(x), y) + |l(f(x), y) - l(h(x), y)|.
\end{equation*}
Note that $y$ corresponds to $x$.
By Lemma \ref{app:lemma}, there exists $C > 0$ satisfying
\begin{equation*}
l(f(x), y) \le l(h(x), y) + C\lVert f(x) - h(x)\rVert_2.
\end{equation*}
By taking the expectation over pairs $(x, y) \sim \mathcal{D}$,
\begin{equation*}
\mathbb{E}[l(f(x), y)] \le \mathbb{E}[l(h(x), y)] + \mathbb{E}[C\lVert f(x)-h(x)\rVert_2].
\end{equation*}
And using the definition of generalization error $R(f) = \mathbb{E}[l(f(x), y)]$, the above inequality is expressed as
\begin{equation*}
R(f) \le R(h) + C \cdot \mathbb{E}[\lVert f(x)-h(x)\rVert_2].
\end{equation*}

By assuming $\mathbb{E}[\lVert f(x) - h(x)\rVert_2] < \epsilon$, we obtain $R(f) \le R(h) + C \cdot \epsilon$.
\hfill $\Box$

\section{\\ Discussion on Sample Size $m$}
\label{app:discussion on sample size}

\begin{theorem}
    Let $\overline{X}_m =\frac{1}{m}\sum_{i=1}^m \|f(x_i) - h(x_i)\|_2$, and let $\mu = \mathbb{E}[\|f(x) - h(x)\|_2]$ for small $\epsilon$ such that $\mu < \epsilon$. To guarantee $|\hat{R}_S(f) - \hat{R}_S(h)| < C \cdot \epsilon$ with probability at least $1-\delta$, the sample size $m$ must satisfy
    \begin{equation*}
    m \ge \frac{n \ln(1/\delta)}{2(\epsilon - \mu)^2}.
    \end{equation*}
\end{theorem}
\textit{Proof.} From the definition of empirical loss, we easily obtain the following inequality.
\begin{equation*}
|\hat{R}_S(f) - \hat{R}_S(h)| \le \frac{1}{m} \sum_{i=1}^m |l(f(x_i), y_i) - l(h(x_i), y_i)|.
\end{equation*}
By Lemma \ref{app:lemma}, the loss function $l$ is Lipschitz continuous. Thus, we obtain the following upper bound:
\begin{equation*}
\frac{1}{m} \sum_{i=1}^m \lVert C \cdot (f(x_i) - h(x_i))\rVert = C \cdot \overline{X}_m.
\end{equation*}
Now, to show $|\hat{R}_S(f) - \hat{R}_S(h)| < C \cdot \epsilon$, it is sufficient to show $\overline{X}_m < \epsilon$.
We aim to find the lower bound of $m$ such that $P(\overline{X}_m \ge \epsilon) \le \delta$. Subtracting the mean $\mu$ from both sides, this is equivalent to the following inequality:
\begin{equation}
P(\overline{X}_m - \mu \ge \epsilon - \mu) \le \delta.
\end{equation}
For the random variable $\|f(x) - h(x)\|_2$, the maximum value is $\sqrt{n}$ and the minimum is $0$. Applying Hoeffding's inequality for bounded random variables, we obtain
\begin{equation}
P(\overline{X}_m - \mu \ge \epsilon - \mu)
\le \exp\left(-\frac{2m(\epsilon - \mu)^2}{n}\right).
\end{equation}
To ensure this probability is at most $\delta$, we set the upper bound to $\delta$:
\begin{equation*}
\exp\left(-\frac{2m(\epsilon - \mu)^2}{n}\right) \le \delta.
\end{equation*}
Solving for m yields the result:
\begin{equation*}
m \ge \frac{n \ln(1/\delta)}{2(\epsilon - \mu)^2}.
\end{equation*}
\hfill $\Box$

\section{\\ Model Configurations for Each Floor}
\label{app:model_configs}

{\footnotesize
Table~\ref{tab:model_configs} summarizes the floor-wise configurations (F=3) used to obtain the results in Tables~\ref{tab:SAL_Performance}--\ref{tab:hyperparam_floor}. Floor~3 is the topmost guiding model and Floor~1 matches the plain baseline. For MLP and LSTM, hidden layers in the same floor share one width. For CNN, channels follow $\lfloor W/d \rfloor \cdot 2^{\lfloor i/2 \rfloor}$, where $W$ is width, $d$ is depth, and $i$ is the layer index.
\par}
\vspace{-0.75\baselineskip}

\begin{table}[H]
\caption{Floor-wise architectural configurations.}
\label{tab:model_configs}
\begin{center}
\tiny
\setlength{\tabcolsep}{2pt}
\renewcommand{\arraystretch}{0.82}
\resizebox{0.9\columnwidth}{!}{\begin{tabular}{llccc}
\toprule
\textbf{Model} & \textbf{Dataset} & \textbf{Floor} & \textbf{Depth} & \textbf{Width} \\
\midrule
MLP & MNIST, CIFAR10        & 3 & 2 & 64  \\
    &                        & 2 & 4 & 128 \\
    &                        & 1 & 8 & 256 \\
\cmidrule{2-5}
    & CIFAR100, IMAGENETTE  & 3 & 2 & 128 \\
    &                        & 2 & 4 & 256 \\
    &                        & 1 & 8 & 512 \\
\midrule
CNN & MNIST, CIFAR10,        & 3 & 2 & 128 \\
    & CIFAR100, IMAGENETTE   & 2 & 4 & 256 \\
    &                        & 1 & 8 & 512 \\
\midrule
LSTM & TREC, TREC\_FINE     & 3 & 2 & 16  \\
     &                       & 2 & 4 & 32  \\
     &                       & 1 & 8 & 64  \\
\cmidrule{2-5}
     & CARER                 & 3 & 2 & 32  \\
     &                       & 2 & 4 & 64  \\
     &                       & 1 & 8 & 128 \\
\bottomrule
\end{tabular}}
\end{center}
\end{table}

\bibliographystyle{IEEEtran}
\bibliography{sal_references}

\begin{IEEEbiography}
[{\includegraphics[width=1in,height=1.25in,clip,keepaspectratio]{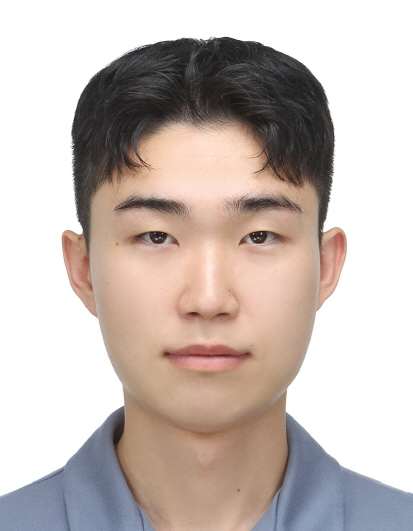}}]{Wonyong Cho} is currently pursuing the B.S. degree in mathematics and artificial intelligence with the University of Seoul, Seoul, South Korea. His research interests include stable training of deep neural networks, AI security, Multi Object Tracking (MOT).
\end{IEEEbiography}

\begin{IEEEbiography}
[{\includegraphics[width=1in,height=1.25in,clip,keepaspectratio]{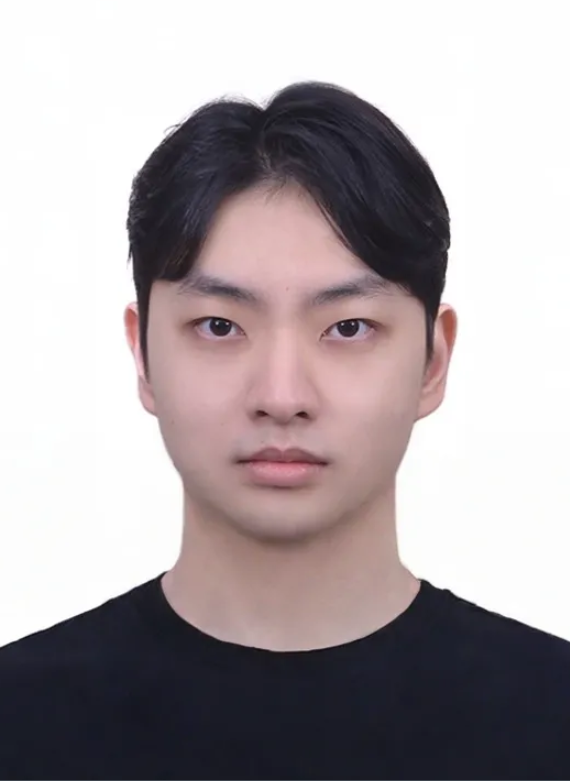}}]{Taemin Kim}
is currently pursuing the B.S. degree in the Division of IT Convergence Engineering at Hansung University, Seoul, Republic of Korea.
His primary research interests are centered on LLMs, particularly in areas of inference optimization, scheduling, and batching. He is also interested in Multi-Modal Learning, Human-Computer Interaction (HCI), and the efficient deployment of on-device AI utilizing NPUs.
\end{IEEEbiography}

\begin{IEEEbiography}[{\includegraphics[width=1in,height=1.25in,clip,keepaspectratio]{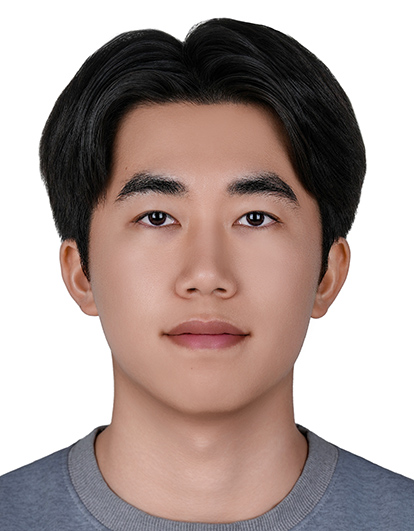}}]{Jungmin Kim} is currently pursuing the B.S. degree in mathematics and artificial intelligence with the University of Seoul, Seoul, South Korea. His research interests include Deep Learning, Artificial Intelligence, Physiological Signal Processing, Medical Image Analysis.
\end{IEEEbiography}

\begin{IEEEbiography}[{\includegraphics[width=1in,height=1.25in,clip,keepaspectratio]{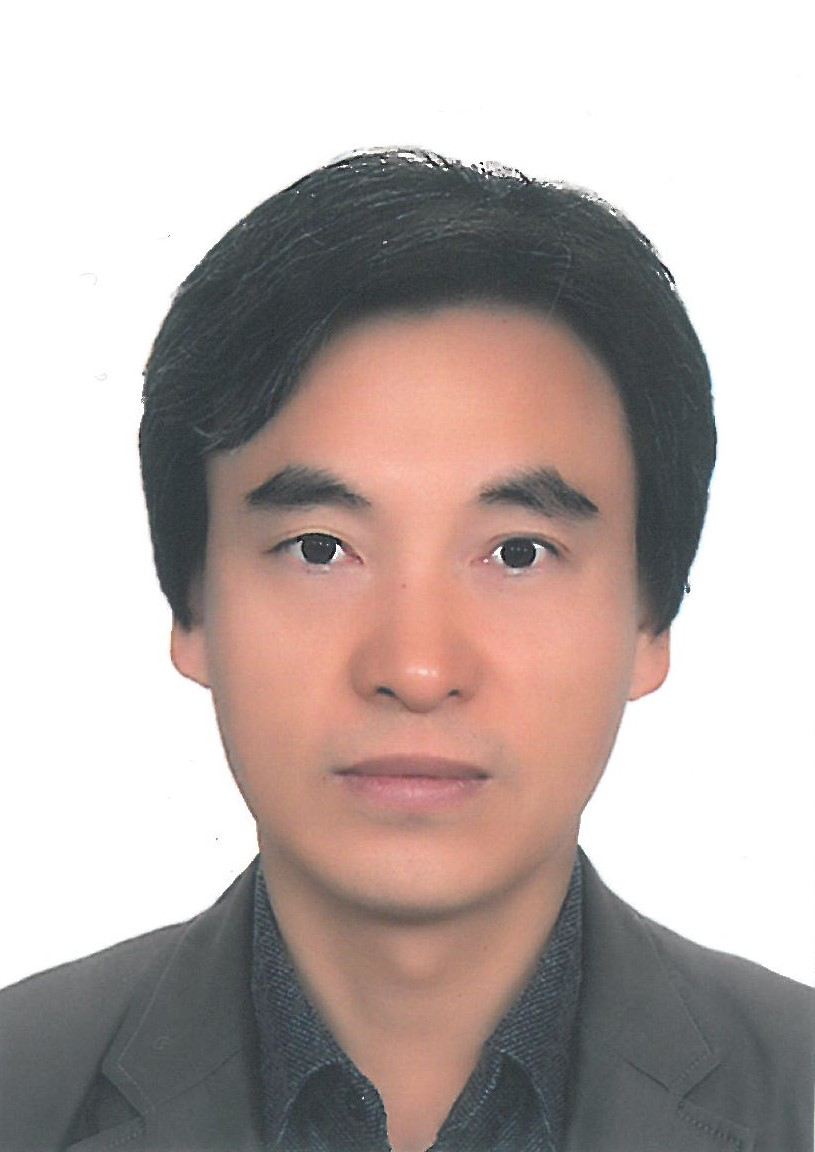}}]{Jeong-Rae Kim}
received the MS and PhD degrees in mathematics from Seoul National University, South Korea, in 1997 and 2004, respectively. From 2005 to 2007, he was with the Bio
Max Institute at Seoul National University as a senior researcher. From 2007 to 2010, he was with the Department of Bio and Brain Engineering at the Korea Advanced Institute of Science and Technology (KAIST) as a research professor. Since 2010, he has been a professor in the Department of Mathematics, University of Seoul. His research interests include systems biology, bioinformatics, numerical analysis, and machine learning.
\end{IEEEbiography}

\begin{IEEEbiography}[{\includegraphics[width=1in,height=1.25in,clip,keepaspectratio]{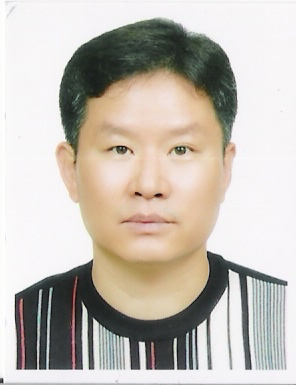}}]{Sung Hoon Jung}
received the B.S. degree in electronics engineering from Hanyang University ERICA Campus, Ansan, Republic of Korea, in 1988, and the M.S. and Ph.D. degrees in electrical engineering from Korea Advanced Institute of Science and Technology (KAIST), Daejeon, Republic of Korea, in 1991 and 1995, respectively. From 1995 to 1996, he was a Research Fellow with KAIST. In 1996, he joined the Department of Information and Communication Engineering, Hansung University, Seoul, where he became a Full Professor, in 2007. From 2009 to 2010, he was a Visiting Professor with the Department of Bio and Brain Engineering, KAIST. Since 2021, he has been with the Department of Applied Artificial Intelligence, Hansung University. His research interests include artificial intelligence, machine learning, and their applications in engineering and technology, particularly in cultural and medical domains.
\end{IEEEbiography}

\EOD
\end{document}